\documentclass[conference]{IEEEtran}

\newtheorem{Them}{Theorem}[section]
\newtheorem{Prob}{Problem}[section]
\newtheorem{Def}{Definition}[section]
\newtheorem{Lemma}{Lemma}[section]
\newtheorem{Prop}{Proposition}[section]

\newtheorem{Assmp}{Assumption}[section]

\newcommand{\st}{\mathrm{s.t.}\quad }

\def \bQ{{\mathbf{Q}}}

\def \bx{{\mathbf{x}}}
\def \by{{\mathbf{y}}}

\usepackage[framemethod=PStricks]{mdframed}
\usepackage{xcolor}
\usepackage{graphicx}
\usepackage{caption}
\usepackage{subcaption}
\usepackage{float}
\usepackage{algorithm}
\usepackage{algorithmic}
\usepackage{colortbl}
\usepackage{tablefootnote}
\usepackage{amsfonts}
\usepackage{mathrsfs}
\usepackage{amsmath}

\DeclareMathOperator*{\argmax}{\arg\!\max}

\definecolor{Gray}{gray}{0.9}

\usepackage[belowskip=-10pt,aboveskip=2pt]{caption}
\usepackage[font=small,labelfont=bf]{caption}

\begin{document}
%
\title{Efficient Approximate Solutions to Mutual Information Based Global Feature Selection}
\author{
	\IEEEauthorblockN{
		Hemanth Venkateswara\IEEEauthorrefmark{1},
		Prasanth Lade\IEEEauthorrefmark{2},
		Binbin Lin\IEEEauthorrefmark{3}, 
		Jieping Ye\IEEEauthorrefmark{3} and
		Sethuraman Panchanathan\IEEEauthorrefmark{1}
	}
	\IEEEauthorblockA{\IEEEauthorrefmark{1}Arizona State University, Tempe, AZ, Emails: \{hemanthv, panch\}@asu.edu}
	\IEEEauthorblockA{\IEEEauthorrefmark{2}Bosch Research and Technology Center, Palo Alto, CA, Email: prasanth.lade@us.bosch.com}
	\IEEEauthorblockA{\IEEEauthorrefmark{3}University of Michigan, Ann Arbor, MI, Emails: binbin.lin@asu.edu and jpye@umich.edu}
}

\maketitle
\begin{abstract}
Mutual Information (MI) is often used for feature selection when developing classifier models. 
Estimating the MI for a subset of features is often intractable. 
We demonstrate, that under the assumptions of conditional independence, MI between a subset of features can be expressed as 
the Conditional Mutual Information (CMI) between pairs of features.
But selecting features with the highest CMI turns out to be a hard combinatorial problem. 
In this work, we have applied two unique global methods, Truncated Power Method (TPower) and Low Rank Bilinear Approximation (LowRank), 
to solve the feature selection problem. 
These algorithms provide very good approximations to the NP-hard CMI based feature selection problem.   
We experimentally demonstrate the effectiveness of these procedures across multiple datasets and 
compare them with existing MI based global and iterative feature selection procedures.
\end{abstract}


%
\IEEEpeerreviewmaketitle

\vspace*{-2mm}
\section{Introduction}
\vspace*{-2mm}
High dimensional data can pose a significant challenge to 
learning methods due to the curse of dimensionality \cite{hastie2009elements}. 
Feature selection is a prominent dimensionality reduction technique that selects a small subset of features 
based on certain relevancy criteria. Apart from reducing data dimensionality, 
feature selection provides insights into the data, prevents over-fitting and reduces computational 
costs for learning, which ultimately results in better learned models. \\
Depending on whether there is label information available, feature selection can be classified into two categories: 
supervised and unsupervised. Supervised feature selection procedures are broadly classified into three groups, \textit{wrapper}, \textit{filter} 
and \textit{embedded} methods \cite{kohavi1997wrappers}. 
Wrapper procedures select features for a specific learning model. 
Filter methods on the other hand are classifier agnostic. 
Feature selection and model learning are treated as two separate steps. These procedures rely on statistical 
characteristics of the data such as correlation, distance and information, to select the most important features. 
Embedded procedures incorporate feature selection as part of the learning model, as seen in neural nets. 
We focus on the model-independent filter procedures for feature selection, because of their classifier 
independence, simplicity and computational efficiency \cite{guyon2003introduction}. 
Specifically, we consider Mutual Information (MI) based criteria for feature selection. 
MI is a probabilistic measure that captures the `correlation' between random variables (see Figure (\ref{Fig1})). 
Whereas standard correlation captures linear relationships between variables, 
MI can capture non-linear dependencies between variables \cite{rodriguez2010quadratic}. \\
Since our aim is to develop better classifier models using feature selection, 
we select the best subset of features that together have the highest 
MI with the class variable. Estimating MI between a subset 
of features requires the estimation of high dimensional joint probability distributions, which in 
turn necessitates exponentially large amounts of data. We circumvent this hurdle by invoking the 
assumption of conditional independence between the features. This reduces the MI estimation problem to a 
Conditional Mutual Information (CMI) estimation problem, involving three features at a time. 
However, selecting the subset of features with the highest CMI turns out to be a very hard 
combinatorial problem. We model feature selection as a Binary Quadratic Problem (BQP), which is NP-hard. 
We introduce approximations to solve the BQP taking inspiration from solutions to other 
related NP-hard problems like $k$-Sparse-PCA and Densest-$k$-Subgraph. 
We also evaluate our procedure by comparing it with less optimal, but computationally more efficient 
iterative procedures for feature selection. We evaluate our feature selection using classification accuracies 
across multiple datasets. \\
\noindent\emph{Motivation}: 
Theoretical underpinnings of using CMI for feature selection have not been clearly outlined in literature. 
Also, current global feature selection methods like Semidefinite Programming (SDP) do not scale well. 
We therefore aim to provide better insights into the approximations and assumptions involved in CMI based feature selection,
and introduce efficient approximations to solving the NP-hard problem 
of feature selection from related problems. \\
\noindent\emph{Contributions}: 
We have demonstrated that the class posterior distribution can be approximated by selecting a subset of variables with the highest MI with the class variable. 
We have proved that the MI between a subset of variables and the class random variable, reduces to CMI between pairs of 
variables under the assumptions of conditional and class-conditional independence. 
We have modeled the NP-hard problem of feature selection as a Binary Quadratic Problem (BQP) and demonstrated our feature selection method across multiple datasets. \\
The rest of our paper is organized as follows. 
In Sec. \ref{model:sec}, we outline the problem of feature selection and formulate the BQP and its approximate solutions. 
Sec. \ref{exist:sec} compares our work with existing feature selection procedures. 
We conclude with experiments in Sec. \ref{expts:sec} followed by discussion in Sec. \ref{discuss:sec}. 

\begin{figure}[!t]
\begin{mdframed}[backgroundcolor=gray!20, roundcorner=10pt, linecolor=blue, linewidth=1pt] 
\begin{center}
  \includegraphics[width=\textwidth]{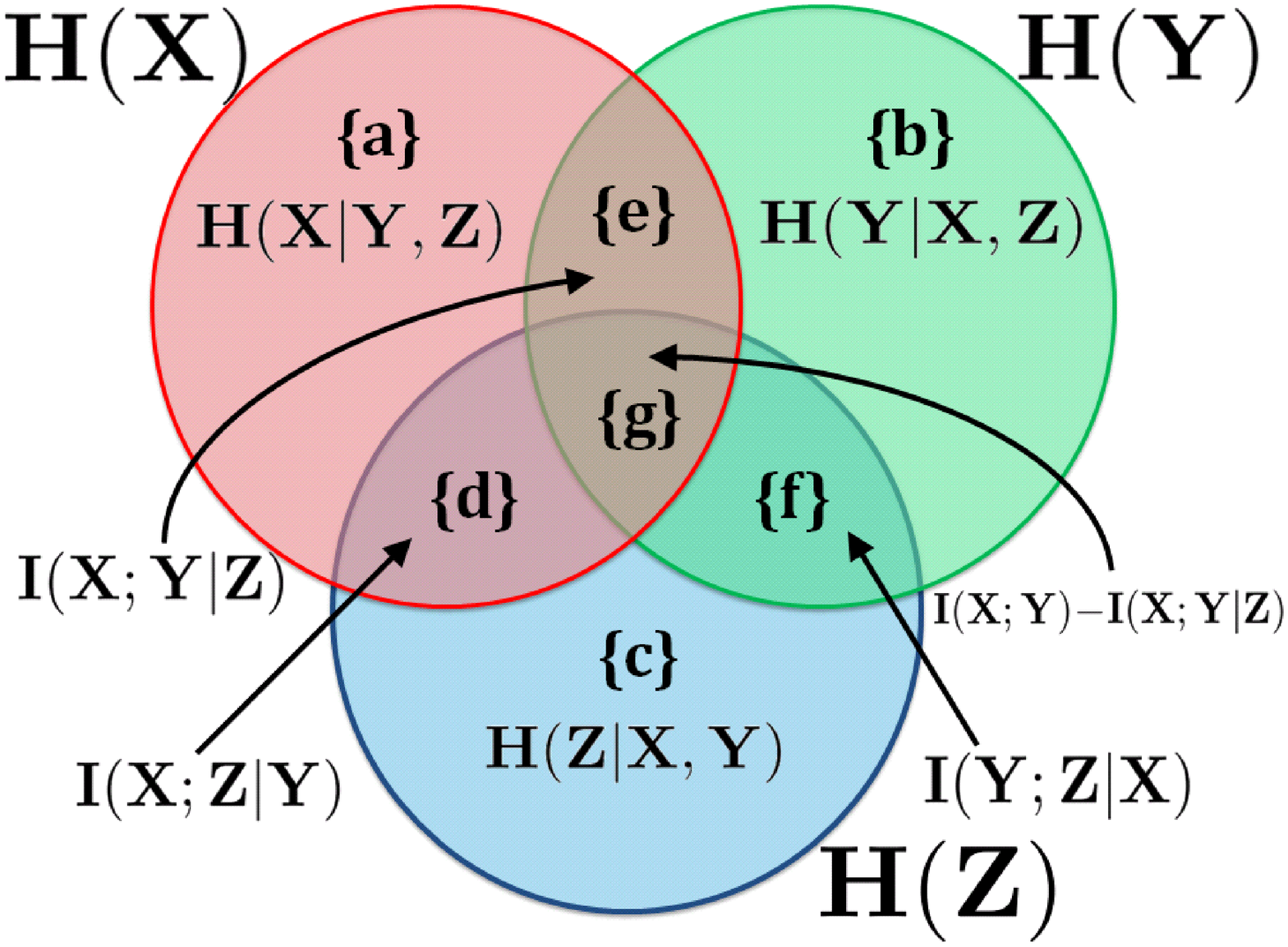}
  \captionof{figure}{Venn diagram depicting entropy interaction. H(X) = \{a,e,g,d\}, H(Y)=\{b,e,g,f\}, 
  H(Z)=\{c,d,g,f\}, I(X;Y)=\{e,g\}, I(X;Z) = \{d,g\}, I(Y;Z) = \{g,f\}, H(X,Y,Z) = \{a,b,c,d,e,f,g\}\\}
  \label{Fig1}
\end{center}
\vspace*{-2mm}
\textbf{\\Information Theory Basics}: 
Information of a random variable $X$ is given by $I(X) = -\log p(X)$. 
Entropy $H(X)$, characterizes the uncertainty about the random variable $X$. 
It is the expected information content for a random variable $X$ with,  
$H(X) = \mathbb{E}[I(X)] = -\sum_{\mathbf{x}}p(X=\mathbf{x})\log p(X=\mathbf{x})$. 
Mutual Information(MI) between two random variables $X$ and $Y$, is a measure of information shared between 
them and is represented as $I(X;Y)$. 
It is symmetric with $I(X;Y) = I(Y;X)$. 
In terms of entropy, MI is defined as $I(X;Y) = H(X) - H(X|Y)$, where $H(X|Y)$ is the conditional entropy. 
MI between random variables $X$ and $Y$ can also be understood as the reduction in entropy of $X$ (or $Y$) 
due to the presence of $Y$ (or $X$). 
Conditional Mutual Information(CMI) denoted as $I(X;Y|Z)$, is the expected MI of two random variables 
$X$ and $Y$, given a third random variable $Z$. 
\end{mdframed}
\vspace*{-4mm}
\end{figure}

\vspace*{-2mm}
\section{Feature Selection using Conditional Mutual Information}
\vspace*{-2mm}
\label{model:sec}
We consider the standard 
classification setting where we are presented with i.i.d data $\mathscr{D}=\{(\mathbf{x}^i,y^i); i=1\ldots m\}$.  
Each data point $\mathbf{x}^i \in \mathbb{R}^n$, is regarded as an instance of a set of $n$ 
continuous random variables $X = \{X_1, X_2,\ldots,X_n\}$. 
We interchangeably use the terms \textit{features} or \textit{variables} to denote $X$ or its subsets. 
The dependent class label $y^i$, is considered to be an instance of a discrete random variable $Y$, that 
takes values in $[1,\ldots,c]$. We define $\mathbb{S}$ to be a subset of $k$ feature indices, where 
$1\leq k\leq n$ and $X_{\mathbb{S}}$, the subset of $k$ features indexed by $\mathbb{S}$. Likewise, 
$\mathbb{\tilde{S}}$ is the subset of left over indices, i.e. $\{1,\ldots,n\}\backslash \mathbb{S}$ and 
$X_{\mathbb{\tilde{S}}}$ is the subset of leftover $(n-k)$ features indexed by $\mathbb{\tilde{S}}$, 
i.e. $X\backslash X_{\mathbb{S}}$. 
We consider $p(Y|X)$, to be the posterior probability on the dataset $\mathscr{D}$. We now state our problem. 
\begin{Prob}
To estimate the optimal subset $\mathbb{S}^* \subset \{1,\ldots, n\}$, of feature indices, such that 
$p(Y|X_{\mathbb{S}^*})$ is a good approximation to $p(Y|X)$.
\end{Prob}
\begin{Them}
\label{MI:Th}
$p(Y|X) = p(Y|X_{\mathbb{S}})$, if and only if Mutual Information $I(X;Y) = I(X_\mathbb{S};Y)$
\end{Them}
\vspace*{-2mm}
\begin{proof} Joint probability $p(X,Y)$ can be factored as
\begin{flalign}
p(X,Y) &= p(X)p(Y|X) = p(X)p(Y|X_{\mathbb{S}}) \notag\\
&= p(X_{\mathbb{\tilde{S}}})p(X_{\mathbb{S}}|X_{\mathbb{\tilde{S}}})p(Y|X_{\mathbb{S}}) \label{Th1:Eq2}
\end{flalign}
When viewed as elements in a Bayesian Network, the random variables can be thought to 
form a Markov Chain $X_{\mathbb{\tilde{S}}} \rightarrow X_{\mathbb{S}} \rightarrow Y$ 
where $Y$ is conditionally independent of $X_{\mathbb{\tilde{S}}}$ given $X_{\mathbb{S}}$. 
Therefore, the CMI $I(X_{\mathbb{\tilde{S}}};Y|X_{\mathbb{S}}) = 0$. 
With $I(X;Y) = I(X_\mathbb{S};Y) + I(X_{\mathbb{\tilde{S}}};Y|X_{\mathbb{S}})$, the necessary condition 
is proved. The statements in the proof are also true in the reverse order.
\end{proof}
Theorem(\ref{MI:Th}) provides the justification for choosing a 
MI based approach to determining the optimal $\mathbb{S}^*$. 
$I(X;Y)$ depends only on $\mathscr{D}$ and is a constant that is shared between $I(X_\mathbb{S};Y)$ 
and $I(X_{\mathbb{\tilde{S}}};Y|X_{\mathbb{S}})$. It is very likely that $I(X_{\mathbb{\tilde{S}}};Y|X_{\mathbb{S}})$ 
is a non-zero value for every $\mathbb{S}$ except for the most trivial case with 
$X_{\mathbb{\tilde{S}}} = \emptyset$. For a fixed value of $k$, we therefore try to estimate the optimal subset 
$\mathbb{S}$ that maximizes $I(X_{\mathbb{S}};Y)$. 

\vspace*{-2mm}
\subsection{The Binary Quadratic Problem}
\vspace*{-2mm}
Given a subset of features $X_\mathbb{S}$, in order to evaluate $I(X_\mathbb{S};Y)$, we need to estimate the 
joint probability distribution $p(X_\mathbb{S},Y)$, which is often intractable. 
To simplify the estimation of the joint distribution, 
we introduce the assumption of conditional independence between 
features in the spirit of Na\"{i}ve Bayes and \cite{brown2012conditional}. 
We first introduce a new term $\mathbb{S}_i$ as the subset $\mathbb{S}$ without the index $i$ also denoted 
as $\mathbb{S}\backslash\{i\}$. We now outline the assumption that will simplify the joint distribution. 
\begin{Assmp}
\label{cond_assmp}
For a set of selected features $\{X_{\mathbb{S}_i} \cup X_i\}$, 
the features $X_{\mathbb{S}_i}$ are conditionally independent and class-conditionally independent given $X_i$, 
i.e. $p(X_{\mathbb{S}_i}|X_i) = \prod_{j\in\mathbb{S}_i}p(X_j|X_i)$ and 
$p(X_{\mathbb{S}_i}|X_i,Y) = \prod_{j\in\mathbb{S}_i}p(X_j|X_i,Y)$. 
\end{Assmp}
\begin{Them}
\label{CMI:Th}
If Assumption \ref{cond_assmp} is true, and $X_i \in X_{\mathbb{S}}$ then, 
$I(X_{\mathbb{S}};Y) =  I(X_i;Y) + \sum_{j\in\mathbb{S}_i}I(X_j;Y|X_i)$
\end{Them}
\vspace*{-2mm}
\begin{proof}
The MI between $X_\mathbb{S}$ and $Y$ is given by:
\begin{flalign}
I(X_{\mathbb{S}};Y) &= I(X_{\mathbb{S}_i},X_i;Y)\notag\\
&= I(X_{\mathbb{S}_i};Y) + I(X_i;Y|X_{\mathbb{S}_i})\notag \\
&= I(X_{\mathbb{S}_i};Y) + I(X_i;Y) \notag \\ 
&\quad- I(X_i;X_{\mathbb{S}_i}) + I(X_i;X_{\mathbb{S}_i}|Y)\notag \\
&= I(X_{\mathbb{S}_i};Y) + I(X_i;Y) - H(X_{\mathbb{S}_i}) + H(X_{\mathbb{S}_i}|X_i) \notag\\
&\quad+ H(X_{\mathbb{S}_i}|Y) - H(X_{\mathbb{S}_i}|X_i,Y)\notag \\
&= I(X_i;Y) + H(X_{\mathbb{S}_i}|X_i) - H(X_{\mathbb{S}_i}|X_i,Y) \label{condMutInfoDer:Eq1}
\end{flalign}
In the above derivation, we have first applied the MI Chain Rule $I(A,B;C) = I(A;C) + I(B;C|A)$. 
We have then applied the MI rule $I(A;B|C)-I(A;B) = I(A;C|B) - I(A;C)$ 
and expressed the two trailing MI terms in terms of entropy.  
Applying Assumption \ref{cond_assmp} to 
(\ref{condMutInfoDer:Eq1}), the high order 
entropy terms can be replaced with summations and reduced further using $H(X_j|X_i) - H(X_j|X_i,Y) = I(X_j,Y|X_i)$, to yield, 
\begin{flalign}
I(X_iX_{\mathbb{S}_i};Y) &\approx I(X_i;Y) + \sum_{j\in\mathbb{S}_i}H(X_j|X_i) \notag \\
&\quad- \sum_{j\in\mathbb{S}_i}H(X_j|X_i,Y)\notag\\
&=I(X_i;Y) + \sum_{j\in\mathbb{S}_i}I(X_j;Y|X_i)
\label{condMutInfoDer:Eq2}
\end{flalign}
\end{proof}
Since we would like to select $\mathbb{S}$ that maximizes (\ref{condMutInfoDer:Eq2}), we can formulate the global feature selection problem as,
\begin{flalign}
\mathbb{S} &= \argmax_{\{\mathbb{S} \mid X_{\mathbb{S}}\subset X\},|\mathbb{S}|=k} \sum_{i\in\mathbb{S}}\big[I(X_i;Y) + \sum_{j\in\mathbb{S}_i}I(X_j;Y|X_i)\big]
\end{flalign}
This is equivalent to the constrained Binary Quadratic problem,
\begin{flalign}
\max_{\mathbf{x}}\{\mathbf{x}^\top \mathbf{Q}\mathbf{x}\} ~~\text{s.t.} ~\mathbf{x}\in\{0,1\}^n, ~||\mathbf{x}||_1 = k,
\tag{BQP}
\label{condQ:Eq}
\end{flalign}
where, $\mathbf{Q}$ is a $[n \times n]$ non-negative matrix with $Q_{ii}$ = $I(X_i;C)$ and 
$Q_{ij}$ = $I(X_j;Y|X_i)$ and
the non-zero indices of the solution $\mathbf{x}$, constitute $\mathbb{S}$. 

\vspace*{-2mm}
\subsection{Solving the Binary Quadratic Problem}
\vspace*{-2mm}
We aim to solve the \ref{condQ:Eq}, where $\mathbf{Q}$ is a symmetric and possibly indefinite matrix with non-negtiave elements, i.e., $Q_{ij} \geq 0$ for all $1 \leq i, j \leq n$. Although this problem is well defined, it is highly nonconvex due to the nonconvex constraint.  And it is also known that this binary quadratic problem is NP-hard~\cite{Garey:1990:CIG}. So, it is difficult to find the optimal value in practice. Therefore, we aim to find an approximate solution to the BQP. Approximation methods such as linear relaxation, spectral relaxation, semidefinite programming, truncated power method and low rank bilinear approximation have been applied to solve this family of NP-hard problems. We will now briefly review some of these relaxation methods.  

\vspace*{-2mm}
\subsubsection{Linear Relaxation}
\label{linr:ssec}
By introducing a new variable $w_{ij}$, we are able to linearize the quadratic term. With $\bx = [x_1,\ldots,x_n]^\top$, we formulate the following:
\begin{equation}
\tag{LP1}
\begin{aligned}
\max_{\bx} &\sum_{i=1}^n\sum_{j=1}^n Q_{ij}w_{ij}, ~~\st~ 2w_{ij} \leq x_i + x_j, ~\sum_{i=1}^n x_i = k,\\
&x_{i}, w_{ij}\in \{ 0 ,1 \}, \quad \forall 1 \leq i, j \leq n. \\
\end{aligned}
\end{equation}
LP1 is equivalent to BQP, which is also NP-hard. 
LP1 is simplified by relaxing $w_{ij}\in [0,1]$. 
One of the optimality conditions is then given by, $2w_{ij} = x_i + x_j$. 
This relaxation reduces LP1 to LP2 which is given by, 
\begin{equation}
\tag{LP2}
\begin{aligned}
\textbf{Linear} = \max_{\bx} &\sum_{i=1}^n\sum_{j=1}^n \frac{1}{2}Q_{ij}(x_i + x_j) = \| \mathbf{Q}\bx \|_1 \\
\st &\|\bx \|_1= k, \bx \in \{ 0 ,1 \}^n. 
\end{aligned}
\end{equation}
Since $Q_{ij}\geq 0$, the maximum value for \textbf{Linear} is equivalent to the $k$ largest column (or row) sum of $\mathbf{Q}$. 
The solution to \textbf{Linear} guarantees a tight lower bound to the BQP (see appendix). 
In our work, we use the solution from \textbf{Linear} to initialize the input to other algorithms. 

\vspace*{-2mm}
\subsubsection{Truncated Power Method}
\label{tpm:ssec}
Truncated power (\textbf{TPower}) method aims to find the largest $k$-sparse eigenvector. Given a positive semidefinite matrix $A$, the largest $k$-sparse eigenvalue can be defined as follows \cite{yuan2013truncated}:
\begin{equation}
\lambda_{\max}(A,k) = \max ~\bx^\top A \bx, ~\st \| \bx \| = 1, \| \bx \|_0 \leq k
\label{TPower:Eq}
\end{equation}
Matrix $A$ is required to be positive semidefinte, but \textbf{TPower} method can be extended to deal with general symmetric matrices by setting $A \leftarrow (A+\tilde\lambda I_{n\times n})$ where $\tilde\lambda > 0$ such that $(A+\tilde\lambda I_{n\times n}) \in \mathbb{S}_+^n$. 
The truncated power method is given as follows. Starting from an initial $k$-sparse vector $\bx_0$, at each iteration $t$, we multiply the vector $\bx_{t-1}$ by $A$ and then truncate the entries of $A\bx_{t-1}$ to zeros and set the largest $k$ entries to 1. \textbf{TPower} can benefit from a good starting point. We use the solution from \textbf{Linear} as the initial sparse vector $\bx_0$. 

\vspace*{-2mm}
\subsubsection{Low Rank Bilinear Approximation}
\label{lrba:ssec}
The low rank bilinear approximation procedure has been applied to solve 
the $k$-Sparse-PCA \cite{papailiopoulos2013sparse} and the Densest-$k$-Subgraph \cite{papailiopoulos2014finding}. 
It approximates the solution to a BQP by applying a bilinear relaxation which is given as, 
\begin{flalign}
\text{BQP}_b =& \max_{\mathbf{x}, \mathbf{y}}\{\mathbf{x}^\top \mathbf{Q}\mathbf{y}\} ~
\text{s.t.} ~\mathbf{x}\in\{0,1\}^n, ~\mathbf{y}\in\{0,1\}^n, \label{lr:Eq}\\
&~||\mathbf{x}||_1 = ||\mathbf{y}||_1 = k \notag
\end{flalign}
$\text{BQP}_b$ provides a good (2$\rho$-approximation \cite{papailiopoulos2014finding}) 
approximation to BQP and can be solved in polynomial time using a $d$-rank approximation of $\bQ$. 
The authors in \cite{papailiopoulos2014finding} have developed the Spannogram algorithm to estimate a candidate set 
(termed the Spannogram $\mathscr{S}$) of vector pairs $(\bx, \by)$ with $k$ features. 
One of the vectors from the pair that maximizes $\text{BQP}_b$, is the bilinear approximate solution to BQP. 

\vspace*{-2mm}
\subsubsection{Related BQP Approximation Methods}
\label{sr:ssec}
For the sake of completeness we briefly mention two other techniques that have been applied to approximate the BQP 
in the domain of feature selection. 
The authors in \cite{nguyen2014effective} propose two relaxation techniques, (1) Spectral relaxation and 
(2) Semidefinite Programming relaxation. In Spectral Relaxation, the constraint on 
the values of $\bx$ are relaxed to continuous values. 
The values of $\bx$ being positive, can be interpreted as feature weights. 
The solution to \textbf{Spectral} has been shown to 
be the largest eigenvector of $\mathbf{Q}$ \cite{nguyen2014effective}. 
Semidefinite Programming relaxation (\textbf{SDP}) 
is a closer approximation to the BQP than \textbf{Spectral}. 
The BQP is approximated by a trace maximization problem using semidefinite relaxation. 
The approximate solution $\bx$, to the BQP, is obtained from the \textbf{SDP} solution through random projection rounding based on 
Cholesky decomposition \cite{goemans1995improved}. For our experiments, we apply random rounding with 100 projections. 
For further details, please refer to \cite{nguyen2014effective}. \\

\vspace*{-2mm}
\section{Existing Methods For MI Based Feature Selection}
\vspace*{-2mm}
\label{exist:sec}
In this section we will discuss existing filter based 
Mutual Information measures for feature selection. We categorize these procedures as 
Greedy (iterative) Selectors and Global Selectors. We limit our discussion to the best feature selectors we found in our studies. 
For a broader perspective on feature selection procedures, we suggest the works of 
Tang et al. \cite{tang2014feature}, Guyon and Elisseeff \cite{guyon2003introduction} and Duch \cite{duch2006filter}. 

\vspace*{-2mm}
\subsection{Greedy Feature Selectors}
\vspace*{-2mm}
Greedy selectors usually begin with an empty set $\mathbb{S}$, and iteratively add to it the most 
important feature indices, until a fixed number of feature indices are selected or a stopping criterion is reached. 
MI between the random variables (features and label) provides the ranking for the features. 
The most basic form of the scoring function is \textit{Maximum Relevance} (\textbf{MaxRel}) \cite{lewis1992feature}, 
where the score is simply the MI between the feature and the class variable. 
To account for the redundancy $I(X_i;X_j)$ between features, Peng et al. \cite{peng2005feature}, introduced the 
\textit{Maximum Relevance Minimum Redundancy} (\textbf{MRMR}) 
criterion, which selects features with maximum relevance to the label and minimum redundancy between each other. 
A greedy procedure very closely related to our technique is the \textit{Joint Mutual Information} (\textbf{JMI}), that
was developed by Yang and Moody \cite{yang1999data}, and later by Meyer et al. \cite{meyer2008information}. 
In our experiments (Sec. \ref{expts:sec}), we evaluate the features selected by these iterative procedures across multiple datasets. 

\vspace*{-2mm}
\subsection{Global Feature Selectors}
\vspace*{-2mm}
There is limited work on MI based global feature selection. 
In \cite{rodriguez2010quadratic}, Rodriguez-Lujan et al. introduced the \textit{Quadratic Programming Feature Selection} (\textbf{QPFS}). 
This method can be viewed as a global alternative to \textbf{MRMR}. 
The second global technique proposed by Nguyen et al. \cite{nguyen2014effective}, is related 
to our method presented in (\ref{condQ:Eq}). 
Nguyen et al. \cite{nguyen2014effective}, model a global feature selection problem  
based on the CMI matrix $\bQ$, and apply \textbf{Spectral} and \textbf{SDP} methods to approximate the solution. 

\vspace*{-2mm}
\section{Experiments}
\vspace*{-2mm}
\label{expts:sec}
We consider two factors to evaluate the feature selection algorithms discussed so far, viz., time complexity and classification 
accuracies. For the global methods, since we are approximating the solution, 
we also consider the tightness of the approximation.  
We conducted our experiments using MATLAB on an Intel Core-i7 2.3 GHz processor with 16GB of memory. 

\vspace*{-2mm}
\subsection{Feature Selectors: a test of scalability}
\vspace*{-2mm}
Greedy feature selectors have time complexities of the order $O(nk)$, which is negligible compared to the time 
complexities of the global feature selectors. Table \ref{time:tab} lists the time complexities for the global algorithms. 
To study time complexities, we conduct multiple experiments $(n,k)$, where we simulate a CMI matrix $\bQ$ by 
a random positive symmetric matrix of size $[n \times n]$, and select $k$ features. 
The time complexity for experiment $(n,k)$, is the average time of convergence over 10 runs.
We use the same set of random matrices for each of the algorithms in the experiment. 
Figure (\ref{timeCurve:Fig}) depicts the convergence times for different experiments. 
\textbf{Linear} algorithm is the most efficient, followed closely by \textbf{Spectral} and \textbf{TPower} methods. 
We used the CVX \cite{cvx} implementation with SDPT3 solver \cite{toh1999sdpt3}, 
for all our \textbf{SDP} experiments. The \textbf{SDP} solver has a huge memory footprint and with matrix sizes $n\geq700$, 
we run into `Out of Memory' errors. 
For the \textbf{LowRank} method, we used the following parameters, $d=3, \epsilon=0.1, \delta=0.1$ for all our experiments. 
Please refer to \cite{papailiopoulos2014finding} for 
more details on the \textbf{LowRank}.
\begin{table}[t]
\centering
\small 
\setlength{\tabcolsep}{7pt} 
\caption[Table caption text]{Time complexities for the Global approximate solutions for BQP in number of features $n$}
\begin{tabular}{| c | c | c | c | c |}
\hline
\textbf{Linear} & \textbf{Spectral} & \textbf{SDP} & \textbf{LowRank}\textsuperscript{*} & \textbf{TPower}\textsuperscript{**} \\
\hline
$O(nk)$ & $O(n^2)$ & $O(n^{4.5})$ & $O(n^{(d+1)})$ & $O(tn^2)$\\
\hline
\multicolumn{5}{l}{\textsuperscript{*}\footnotesize{$d$ is approximation rank}, 
~\textsuperscript{**}\footnotesize{$t$ is number of iterations}}\\
\end{tabular}
\label{time:tab}
\end{table}

\begin{figure}[!t]
\centering
\includegraphics[trim = 8mm 0mm 15mm 8mm, clip, width=0.49\textwidth, height=2.0in]{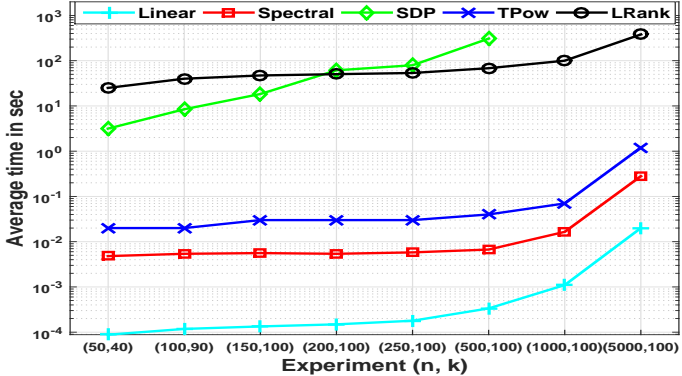}
\caption{Average time in seconds for an algorithm to select $k$ features from data containing $n$ features in
experiment $(n,k)$.}
\label{timeCurve:Fig}
\end{figure}  

\vspace*{-2mm}
\subsection{BQP Methods: a test of approximation}
\vspace*{-2mm}
\begin{figure}[!t]
\centering
\includegraphics[trim = 0mm 0mm 1mm 0mm, clip, width=0.49\textwidth, height=2.0in]{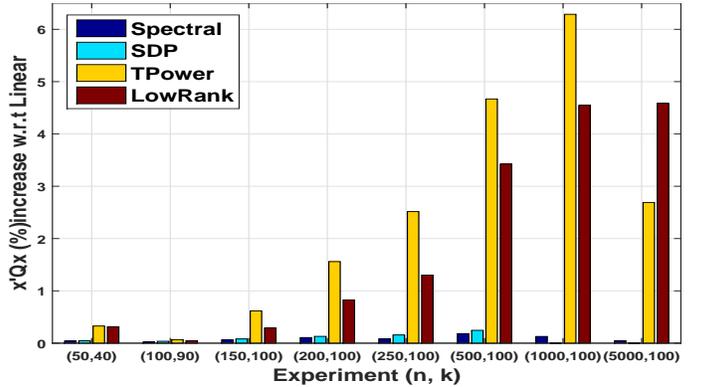}
\caption{The average percentage difference of the BQP objective values compared with the \textbf{Linear} BQP objective value.  
In experiment $(n,k)$, $n$ is the matrix dimension and $k$ is the number of features selected.}
\label{objBar:Fig}
\end{figure}

For the next set of experiments, we evaluated the 
degree of approximation for the global algorithms. Since we do not know the optimal solution 
for the BQP, we compared the methods on their relative objective values. 
We estimated the binary feature vector $\mathbf{x}$ after applying each of the methods and then evaluated the 
objective value $\mathbf{x}^\top\mathbf{Q}\mathbf{x}$. 
We evaluated the percentage difference of every algorithm's objective value with the \textbf{Linear} method's objective. 
Similar to the experimental evaluation used for time complexity, 
we generated random data for each experiment $(n,k)$ and  
averaged the values over 10 random runs.
Figure (\ref{objBar:Fig}) presents the results of the experiment. \textbf{TPower} and 
\textbf{LowRank} displayed the largest percentage increase from the \textbf{Linear}. 
Since \textbf{TPower} and \textbf{LowRank} approximate the BQP better than other methods, they must therefore 
be better feature selectors compared to \textbf{Linear}, \textbf{Spectral} and \textbf{SDP}. The \textbf{TPower} is also 
very efficient in terms of execution time and would be the ideal feature selector when considering both speed and accuracy.

\begin{table}[b]
\centering
\small 
\setlength{\tabcolsep}{10pt} 
\caption[Table caption text]{Datasets details: $n$ is number of features, $m$ is number of samples, $c$ is number of categories, 
\textbf{Error}: is average cross validation error (\%) using all features.}
\begin{tabular}{ l  c  c  c  c  c }
\rowcolor{Gray}
\multicolumn{1}{c}{\textbf{Data}}   & \multicolumn{1}{c}{$n$} & \multicolumn{1}{c}{$m$} & \multicolumn{1}{c}{$c$} & \multicolumn{1}{c}{\textbf{Error}} 	 & \multicolumn{1}{c}{\textbf{Ref.}} \\
\hline
Arrhythmia & 258 & 420 & 2 &  31.1  & \cite{Lichman:2013}\\
Colon & 2000 & 62 &2 &  37.0   & \cite{ding2005minimum}\\
Gisette & 4995 & 6000 & 2 &  2.5  & \cite{Lichman:2013}\\
Leukemia & 7070 & 72 & 2 &  26.4   & \cite{ding2005minimum}\\
Lung & 325 & 73 & 7 &  9.6   & \cite{ding2005minimum}\\
Lymphoma & 4026 & 96 & 9 &  81.3  & \cite{ding2005minimum}\\
Madelon & 500 & 2000 & 2 &  45.5   & \cite{Lichman:2013}\\
Multi-Feat & 649 & 2000 & 10 &  1.5   & \cite{Lichman:2013}\\
Musk2 & 166 & 6598 & 2 &  4.6  & \cite{Lichman:2013}\\
OptDigits & 62 & 3823 & 10 &  3.3  & \cite{Lichman:2013}\\
Promoter & 57 & 106 & 2 &  26.0  & \cite{Lichman:2013}\\
Spambase & 57 & 4601 & 2 &  7.5  & \cite{Lichman:2013}\\
Waveform & 21 & 5000 & 3 &  13.1  & \cite{Lichman:2013}\\
\hline
\end{tabular}
\label{datasets:tab}
\end{table}

\vspace*{-2mm}
\subsection{Feature Selectors: a test of classification error}
\vspace*{-2mm}
In this section we compare the \textbf{TPower} and the \textbf{LowRank} with other algorithms in terms of classification accuracies. 
For our experiments, we chose 13 publicly available 
datasets that are widely used to study MI based feature selection as in 
\cite{rodriguez2010quadratic, brown2012conditional, nguyen2014effective, peng2005feature}. 
The details of the datasets are captured in Table \ref{datasets:tab}. 
We performed feature selection for a set of $k$ values and estimated the classifier 
performance across all values of $k$. Starting at $k=10$, we incremented in steps of $1$ till $n$ or $100$, whichever is smaller. 
We evaluated the classifier performance using Leave-One-Out cross validation (if $m\leq 100$) or 
10-fold cross validation and obtained the cross validation errors(\%) for each fold. 
Since the average error across all values of $k$ 
is not a good measure of the classifier performance, we applied the paired t-test, as also mentioned in 
\cite{rodriguez2010quadratic, nguyen2014effective, herman2013mutual}, across the cross validation folds. 
For a fixed dataset and a fixed value of $k$, to compare \textbf{TPower} 
with say, \textbf{MaxRel}, we applied the one sided paired t-test at 5\% significance over the error(\%) 
of the cross validation folds for the two algorithms. We set the performance of \textbf{TPower} vs. \textbf{MaxRel} 
to win = $w$, tie = $t$ and loss = $l$, based on the largest number of t-test decisions across all the $k$ values. 
Along the lines of earlier studies in feature selection, we used linear SVM as the classifier. 
To estimate the CMI we need to discretize the features. We believe the role of discretization is not unduly critical as long as it is consistent across all the experiments. We discretized the features using the Class Attribute Interdependence Maximization (CAIM) algorithm 
developed by Kurgan and Cios et al. \cite{kurgan2004caim}. 
Feature selection was performed on discretized data but 
the classification (after feature selection) was performed on the original feature space. 
Using the above procedure, we compared the performance of \textbf{TPower} and \textbf{LowRank} with all the other algorithms. 
Tables \ref{tpower:tab} and \ref{lowrank:tab} display the results of the experiment. 
The values in Table \ref{tpower:tab} (likewise \ref{lowrank:tab}) correspond 
to the difference in the average of classification error(\%) between \textbf{TPower} (likewise \textbf{LowRank}) and all the other algorithms. 
From the results in these tables, we find that \textbf{TPower} and \textbf{LowRank} do well on most of the datasets across all the algorithms. 
When compared against each other 
\textbf{LowRank} does better than \textbf{TPower}. 
For the sake of brevity we have not displayed the comparisons between other pairs of algorithms. 
The win/tie/loss numbers by themselves do not provide a complete picture of the comparison. 
The difference in the average error also needs to be taken into account to assess the performance. A large percentage of negative values in the columns and their magnitudes indicate the low error values in classification for \textbf{TPower} and \textbf{LowRank}. 
Figure (\ref{fig:experiments}) displays the average classification error(\%) trends for varying values of $k$ for $3$ datasets. 
Figures (\ref{fig:greedyColon}) and (\ref{fig:globalColon}) for the Colon dataset, suggest that the addition of more 
features does not necessarily reduce classification error. 
Classification error trends also help us cross validate the best value of $k$ for a dataset. 
For a given dataset,  the error trends between the global and greedy procedures follow a similar pattern. This perhaps indicates that 
nearly similar features are being selected using both types of methods. We also note that for huge datasets with large values of $n$, greedy methods 
may not be a bad choice for feature selection. 

\begin{table*}[t]
\centering
\small 
\setlength{\tabcolsep}{6pt} 
\caption[Table caption text]{Comparison of $\textbf{TPower}$ with other algorithms. The table values measure the difference in 
average classification accuracies of \textbf{TPower} with other algorithms. $w$, $t$ and $l$ indicate one-sided paired t-test results. 
The last row displays the total number of Wins($W$), Ties($T$) and Loss($L$). 
N/A indicates comparison data was unavailable for large datasets using \textbf{SDP}.}
\begin{tabular}{l c c c c c c c}
\rowcolor{Gray}
\multicolumn{1}{c}{\textbf{Data}}   & \multicolumn{1}{c}{\textbf{MaxRel}} & \multicolumn{1}{c}{\textbf{MRMR}} & \multicolumn{1}{c}{\textbf{JMI}} & \multicolumn{1}{c}{\textbf{QPFS}} & \multicolumn{1}{c}{\textbf{Spectral}} & \multicolumn{1}{c}{\textbf{SDP}} & \multicolumn{1}{c}{\textbf{LowRank}}\\ \hline\hline
Arrythmia	&	-0.37	$\pm$	1.4	$^{	t	}$	&	0.32	$\pm$	1.0	$^{	l	}$	&	0.02	$\pm$	1.0	$^{	l	}$	&	0.20	$\pm$	1.8	$^{	l	}$	&	-0.08	$\pm$	1.1	$^{	t	}$	&	-0.18	$\pm$	1.0	$^{	w	}$	&	-0.05	$\pm$	0.8	$^{	l	}$	\\
Colon	&	-7.28	$\pm$	4.6	$^{	w	}$	&	-4.42	$\pm$	4.2	$^{	w	}$	&	-2.47	$\pm$	3.8	$^{	w	}$	&	-6.70	$\pm$	4.6	$^{	w	}$	&	-0.60	$\pm$	2.8	$^{	w	}$	&	N/A	&	4.03	$\pm$	4.5	$^{	l	}$	\\
Gisette	&	-1.32	$\pm$	0.6	$^{	w	}$	&	0.00	$\pm$	0.7	$^{	w	}$	&	-1.12	$\pm$	0.6	$^{	w	}$	&	-1.38	$\pm$	0.7	$^{	w	}$	&	-1.26	$\pm$	0.6	$^{	w	}$	&	N/A	&	0.33	$\pm$	0.6	$^{	l	}$	\\
Leukemia	&	0.11	$\pm$	1.4	$^{	w	}$	&	1.40	$\pm$	1.6	$^{	l	}$	&	1.59	$\pm$	1.8	$^{	l	}$	&	0.41	$\pm$	1.1	$^{	t	}$	&	-0.03	$\pm$	0.6	$^{	w	}$	&	N/A	&	1.49	$\pm$	1.3	$^{	l	}$	\\
Lung	&	-9.43	$\pm$	4.1	$^{	w	}$	&	-2.52	$\pm$	4.2	$^{	w	}$	&	-3.83	$\pm$	4.2	$^{	w	}$	&	0.60	$\pm$	2.8	$^{	l	}$	&	-0.88	$\pm$	2.2	$^{	w	}$	&	-0.88	$\pm$	2.1	$^{	w	}$	&	-1.59	$\pm$	2.4	$^{	w	}$	\\
Lymphoma	&	-2.76	$\pm$	4.8	$^{	w	}$	&	3.35	$\pm$	4.7	$^{	l	}$	&	2.93	$\pm$	5.3	$^{	l	}$	&	4.99	$\pm$	3.3	$^{	l	}$	&	-1.86	$\pm$	2.5	$^{	w	}$	&	N/A	&	3.29	$\pm$	4.2	$^{	l	}$	\\
Madelon	&	0.32	$\pm$	0.5	$^{	l	}$	&	0.80	$\pm$	0.9	$^{	l	}$	&	0.01	$\pm$	0.4	$^{	w	}$	&	-0.22	$\pm$	0.7	$^{	w	}$	&	-0.01	$\pm$	0.6	$^{	w	}$	&	0.15	$\pm$	0.6	$^{	l	}$	&	-0.11	$\pm$	0.4	$^{	t	}$	\\
MultiFeatures	&	0.02	$\pm$	0.3	$^{	w	}$	&	0.24	$\pm$	0.3	$^{	l	}$	&	0.17	$\pm$	0.3	$^{	l	}$	&	-0.42	$\pm$	0.3	$^{	w	}$	&	0.10	$\pm$	0.3	$^{	w	}$	&	0.11	$\pm$	0.3	$^{	w	}$	&	0.01	$\pm$	0.3	$^{	l	}$	\\
Musk2	&	-0.45	$\pm$	0.6	$^{	w	}$	&	-0.22	$\pm$	0.7	$^{	w	}$	&	-0.18	$\pm$	0.5	$^{	w	}$	&	-0.31	$\pm$	0.6	$^{	w	}$	&	0.06	$\pm$	0.4	$^{	w	}$	&	0.03	$\pm$	0.5	$^{	w	}$	&	0.05	$\pm$	0.4	$^{	w	}$	\\
OptDigits	&	-0.19	$\pm$	0.5	$^{	w	}$	&	-0.01	$\pm$	0.6	$^{	t	}$	&	0.16	$\pm$	0.6	$^{	l	}$	&	-0.65	$\pm$	1.0	$^{	w	}$	&	0.03	$\pm$	0.3	$^{	l	}$	&	-2.53	$\pm$	13.0	$^{	w	}$	&	0.08	$\pm$	0.4	$^{	l	}$	\\
Promoter	&	0.73	$\pm$	3.0	$^{	l	}$	&	-0.04	$\pm$	3.2	$^{	w	}$	&	0.19	$\pm$	3.0	$^{	l	}$	&	-1.29	$\pm$	3.8	$^{	w	}$	&	-0.48	$\pm$	2.8	$^{	w	}$	&	-0.56	$\pm$	2.9	$^{	w	}$	&	-0.17	$\pm$	3.1	$^{	w	}$	\\
Spambase	&	-0.34	$\pm$	0.3	$^{	w	}$	&	0.06	$\pm$	0.2	$^{	l	}$	&	-0.23	$\pm$	0.3	$^{	w	}$	&	0.03	$\pm$	0.4	$^{	l	}$	&	-0.09	$\pm$	0.3	$^{	w	}$	&	-0.10	$\pm$	0.3	$^{	w	}$	&	0.02	$\pm$	0.1	$^{	l	}$	\\
Waveform	&	-0.13	$\pm$	0.3	$^{	w	}$	&	0.06	$\pm$	0.3	$^{	l	}$	&	-0.01	$\pm$	0.0	$^{	t	}$	&	0.04	$\pm$	0.2	$^{	t	}$	&	0.04	$\pm$	0.1	$^{	t	}$	&	0.00	$\pm$	0.2	$^{	t	}$	&	-0.01	$\pm$	0.2	$^{	t	}$	\\ \hline

\multicolumn{1}{c}{\#$W$/$T$/$L$:} & \multicolumn{1}{c}{10/1/2} & \multicolumn{1}{c}{5/1/7} & \multicolumn{1}{c}{6/1/6} & \multicolumn{1}{c}{7/2/4} & \multicolumn{1}{c}{10/2/1} & \multicolumn{1}{c}{7/1/1} & \multicolumn{1}{c}{3/2/8}  \\\hline
 
\end{tabular}
\label{tpower:tab}
\end{table*}

\begin{table*}[t]
\centering
\small 
\setlength{\tabcolsep}{6pt} 
\caption[Table caption text]{Comparison of $\textbf{LowRank}$ with other algorithms. Table structure similar to Table \ref{tpower:tab}}
\begin{tabular}{l c c c c c c c}
\rowcolor{Gray}
\multicolumn{1}{c}{\textbf{Data}}   & \multicolumn{1}{c}{\textbf{MaxRel}} & \multicolumn{1}{c}{\textbf{MRMR}} & \multicolumn{1}{c}{\textbf{JMI}} & \multicolumn{1}{c}{\textbf{QPFS}} & \multicolumn{1}{c}{\textbf{Spectral}} & \multicolumn{1}{c}{\textbf{SDP}} & \multicolumn{1}{c}{\textbf{TPower}}\\ \hline\hline
Arrythmia	&	-0.32	$\pm$	1.3	$^{	l	}$	&	0.36	$\pm$	1.0	$^{	l	}$	&	0.07	$\pm$	1.0	$^{	l	}$	&	0.25	$\pm$	1.7	$^{	l	}$	&	-0.03	$\pm$	1.1	$^{	t	}$	&	-0.13	$\pm$	1.0	$^{	w	}$	&	0.05	$\pm$	0.8	$^{	w	}$	\\
Colon	&	-11.31	$\pm$	4.7	$^{	w	}$	&	-8.45	$\pm$	4.3	$^{	w	}$	&	-6.50	$\pm$	3.9	$^{	w	}$	&	-10.73	$\pm$	5.3	$^{	w	}$	&	-4.63	$\pm$	5.0	$^{	w	}$	&	N/A	&	-4.03	$\pm$	4.5	$^{	w	}$	\\
Gisette	&	-1.65	$\pm$	0.5	$^{	w	}$	&	-0.32	$\pm$	0.5	$^{	w	}$	&	-1.44	$\pm$	0.7	$^{	w	}$	&	-1.70	$\pm$	0.6	$^{	w	}$	&	-1.58	$\pm$	0.6	$^{	w	}$	&	N/A	&	-0.33	$\pm$	0.6	$^{	w	}$	\\
Leukemia	&	-1.39	$\pm$	1.4	$^{	w	}$	&	-0.09	$\pm$	1.5	$^{	t	}$	&	0.09	$\pm$	1.7	$^{	t	}$	&	-1.09	$\pm$	1.3	$^{	w	}$	&	-1.52	$\pm$	1.2	$^{	w	}$	&	N/A	&	-1.49	$\pm$	1.3	$^{	w	}$	\\
Lung	&	-7.83	$\pm$	4.1	$^{	w	}$	&	-0.92	$\pm$	4.5	$^{	w	}$	&	-2.23	$\pm$	4.5	$^{	w	}$	&	2.19	$\pm$	3.7	$^{	l	}$	&	0.71	$\pm$	2.5	$^{	l	}$	&	0.71	$\pm$	2.3	$^{	l	}$	&	1.59	$\pm$	2.4	$^{	l	}$	\\
Lymphoma	&	-6.06	$\pm$	3.5	$^{	w	}$	&	0.06	$\pm$	2.1	$^{	l	}$	&	-0.36	$\pm$	2.1	$^{	l	}$	&	1.70	$\pm$	2.5	$^{	l	}$	&	-5.15	$\pm$	3.8	$^{	w	}$	&	N/A	&	-3.29	$\pm$	4.2	$^{	w	}$	\\
Madelon	&	0.43	$\pm$	0.5	$^{	l	}$	&	0.91	$\pm$	0.8	$^{	l	}$	&	0.12	$\pm$	0.5	$^{	l	}$	&	-0.11	$\pm$	0.7	$^{	w	}$	&	0.10	$\pm$	0.6	$^{	l	}$	&	0.26	$\pm$	0.5	$^{	l	}$	&	0.11	$\pm$	0.4	$^{	t	}$	\\
MultiFeatures	&	0.01	$\pm$	0.4	$^{	l	}$	&	0.23	$\pm$	0.3	$^{	l	}$	&	0.16	$\pm$	0.4	$^{	l	}$	&	-0.43	$\pm$	0.3	$^{	w	}$	&	0.10	$\pm$	0.4	$^{	l	}$	&	0.10	$\pm$	0.4	$^{	l	}$	&	-0.01	$\pm$	0.3	$^{	w	}$	\\
Musk2	&	-0.50	$\pm$	0.4	$^{	w	}$	&	-0.27	$\pm$	0.7	$^{	w	}$	&	-0.23	$\pm$	0.5	$^{	w	}$	&	-0.36	$\pm$	0.5	$^{	w	}$	&	0.02	$\pm$	0.3	$^{	l	}$	&	-0.02	$\pm$	0.4	$^{	l	}$	&	-0.05	$\pm$	0.4	$^{	l	}$	\\
OptDigits	&	-0.26	$\pm$	0.6	$^{	w	}$	&	-0.09	$\pm$	0.4	$^{	w	}$	&	0.08	$\pm$	0.4	$^{	l	}$	&	-0.72	$\pm$	1.2	$^{	w	}$	&	-0.04	$\pm$	0.4	$^{	t	}$	&	-2.61	$\pm$	13.1	$^{	w	}$	&	-0.08	$\pm$	0.4	$^{	w	}$	\\
Promoter	&	0.90	$\pm$	3.1	$^{	l	}$	&	0.13	$\pm$	2.7	$^{	w	}$	&	0.35	$\pm$	2.9	$^{	w	}$	&	-1.13	$\pm$	3.7	$^{	w	}$	&	-0.31	$\pm$	2.8	$^{	t	}$	&	-0.40	$\pm$	2.5	$^{	t	}$	&	0.17	$\pm$	3.1	$^{	l	}$	\\
Spambase	&	-0.36	$\pm$	0.3	$^{	w	}$	&	0.04	$\pm$	0.2	$^{	l	}$	&	-0.24	$\pm$	0.3	$^{	w	}$	&	0.02	$\pm$	0.4	$^{	t	}$	&	-0.11	$\pm$	0.3	$^{	w	}$	&	-0.12	$\pm$	0.3	$^{	w	}$	&	-0.02	$\pm$	0.1	$^{	w	}$	\\
Waveform	&	-0.12	$\pm$	0.4	$^{	w	}$	&	0.07	$\pm$	0.1	$^{	l	}$	&	0.00	$\pm$	0.2	$^{	t	}$	&	0.05	$\pm$	0.1	$^{	t	}$	&	0.05	$\pm$	0.1	$^{	t	}$	&	0.02	$\pm$	0.1	$^{	t	}$	&	0.01	$\pm$	0.2	$^{	t	}$	\\ \hline

\multicolumn{1}{c}{\#$W$/$T$/$L$:} & \multicolumn{1}{c}{9/0/4} & \multicolumn{1}{c}{6/1/6} & \multicolumn{1}{c}{6/2/5} & \multicolumn{1}{c}{8/2/3} & \multicolumn{1}{c}{5/4/4} & \multicolumn{1}{c}{3/2/4} & \multicolumn{1}{c}{8/2/3}  \\\hline
 
\end{tabular}
\label{lowrank:tab}
\end{table*}

\setlength{\belowcaptionskip}{3pt}
\begin{figure*}[!t]
        \centering
        \begin{subfigure}[b]{0.3\textwidth}
                \includegraphics[trim = 5mm 0mm 3mm 7mm, clip, width=2.3in, height=1.5in]{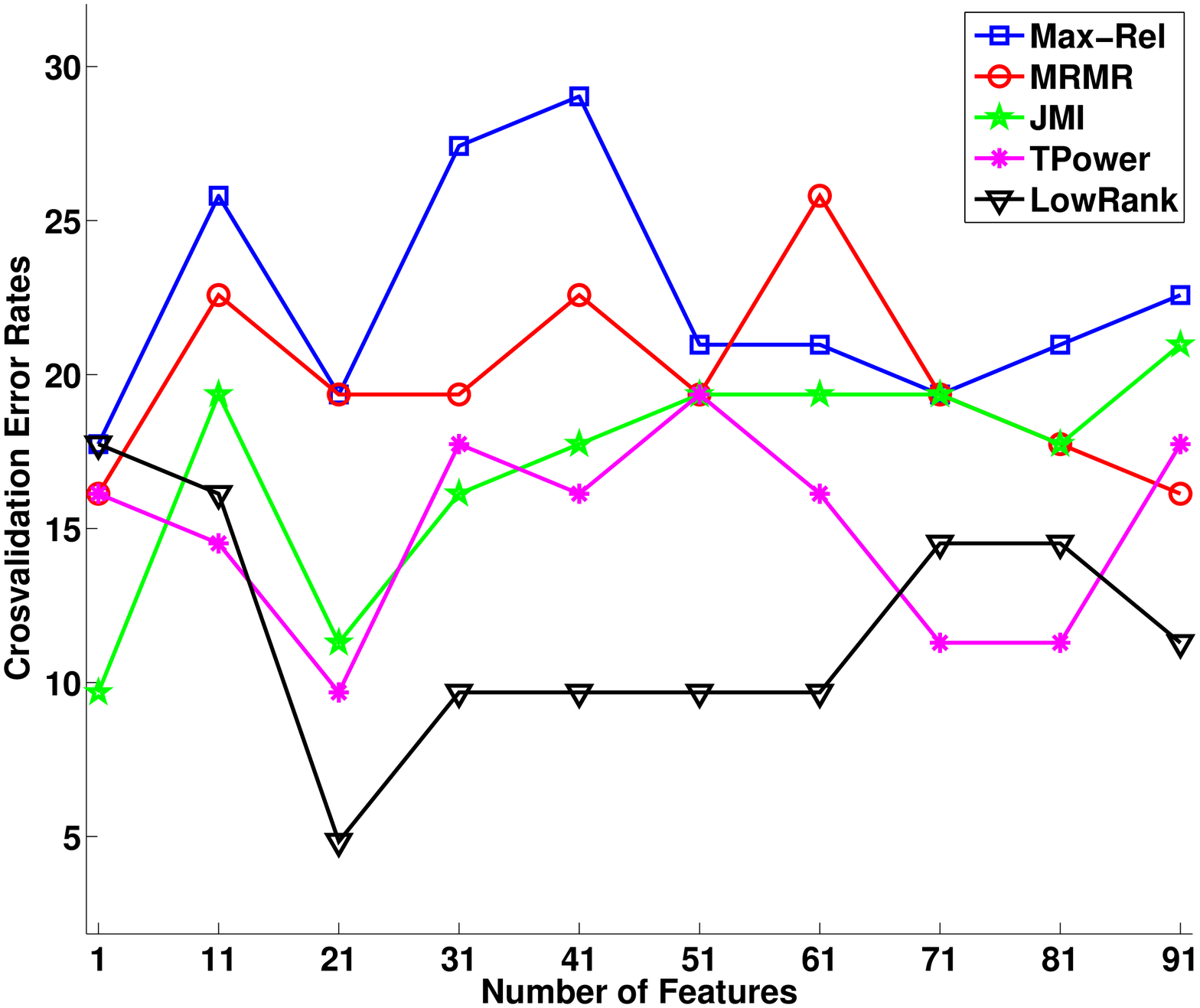}
                \caption{Colon}
                \label{fig:greedyColon}
        \end{subfigure}%
        \quad 
        \begin{subfigure}[b]{0.3\textwidth}
                \includegraphics[trim = 5mm 0mm 3mm 7mm, clip, width=2.3in, height=1.5in]{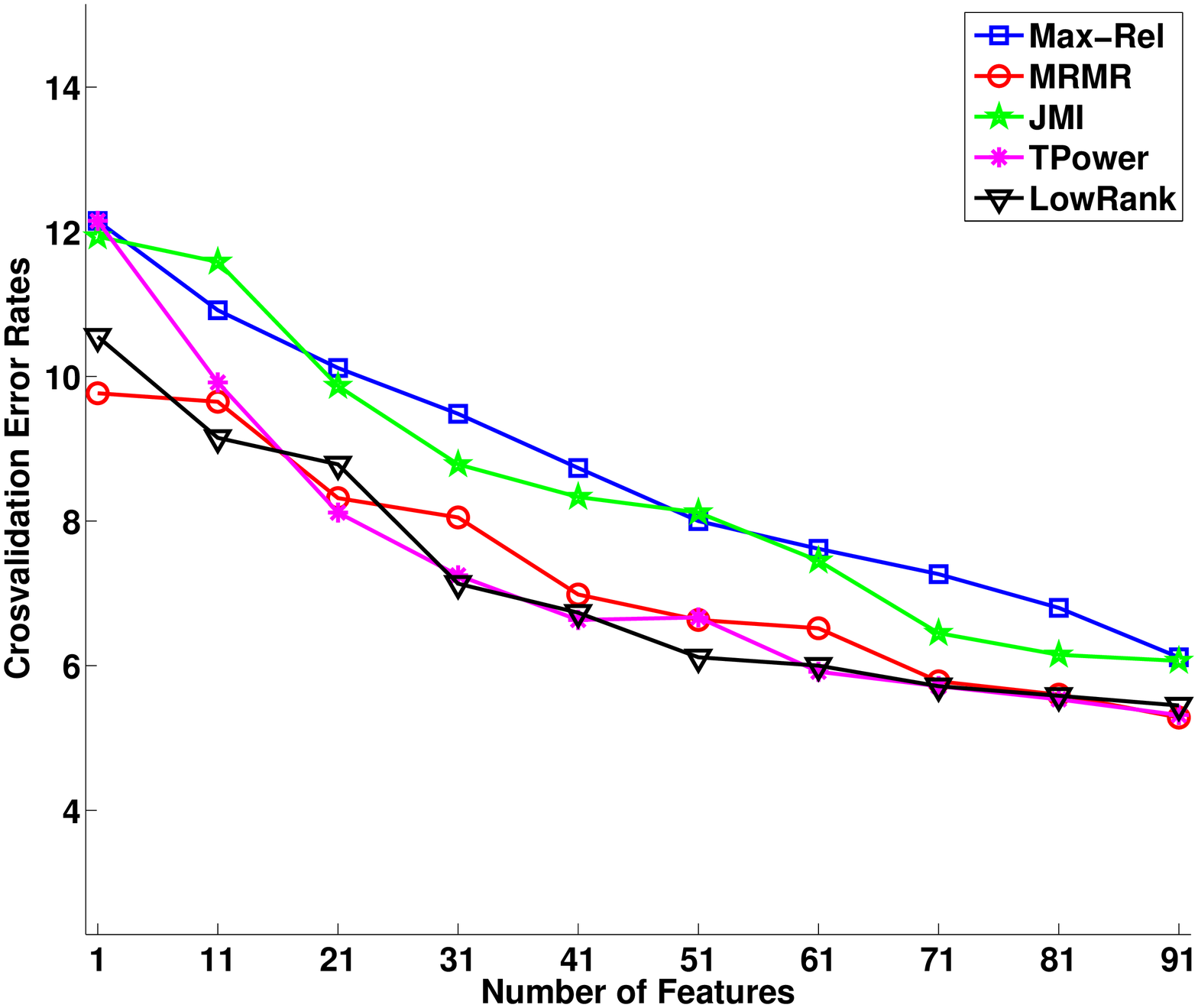}
                \caption{Gisette}
                \label{fig:greedyGisette}
        \end{subfigure}
        \quad 
        \begin{subfigure}[b]{0.3\textwidth}
                \includegraphics[trim = 5mm 0mm 3mm 7mm, clip, width=2.3in, height=1.5in]{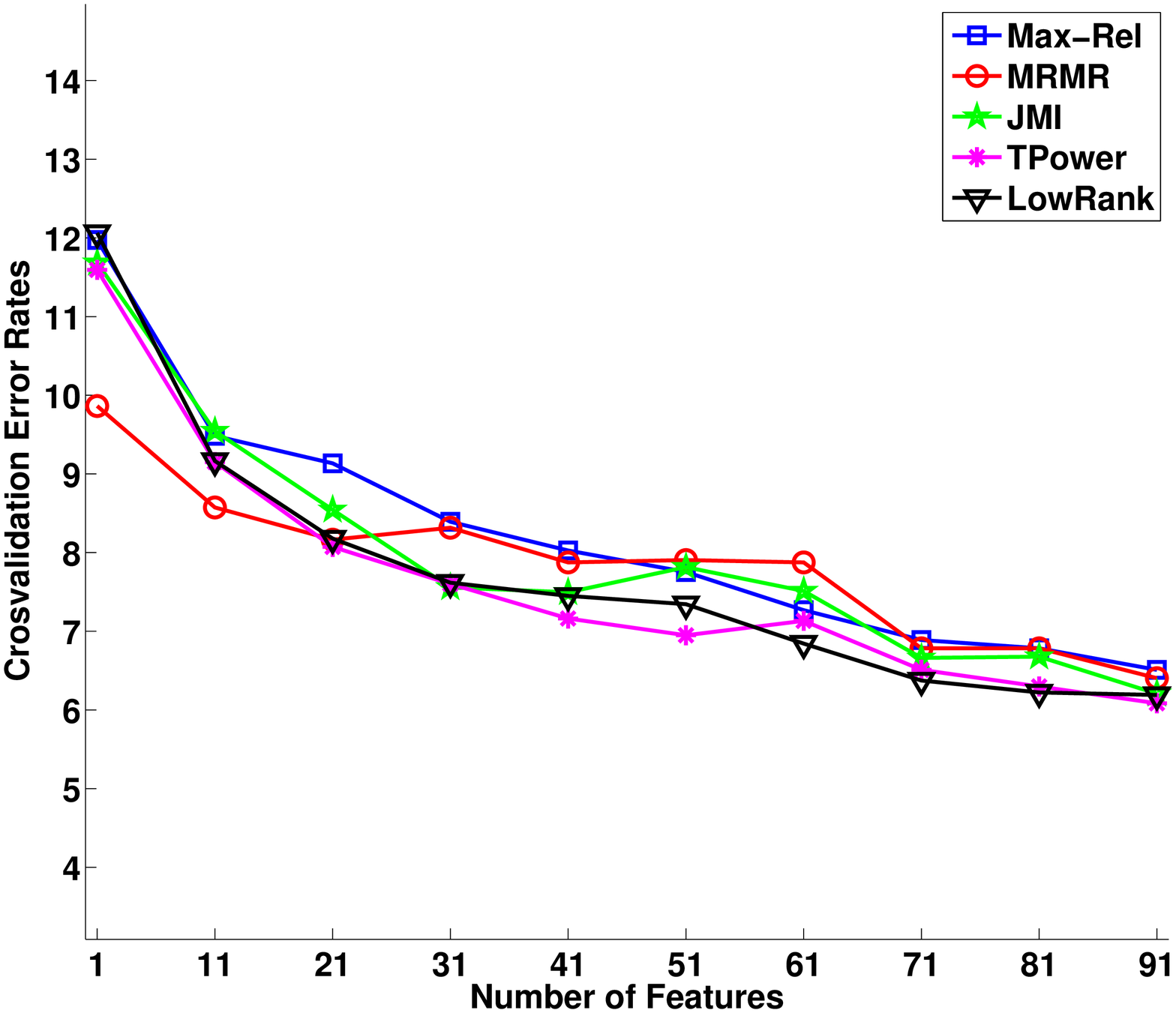}
                \caption{Musk2}
                \label{fig:greedyMusk2}
        \end{subfigure}
        \begin{subfigure}[b]{0.3\textwidth}
                \includegraphics[trim = 5mm 0mm 3mm 7mm, clip, width=2.3in, height=1.5in]{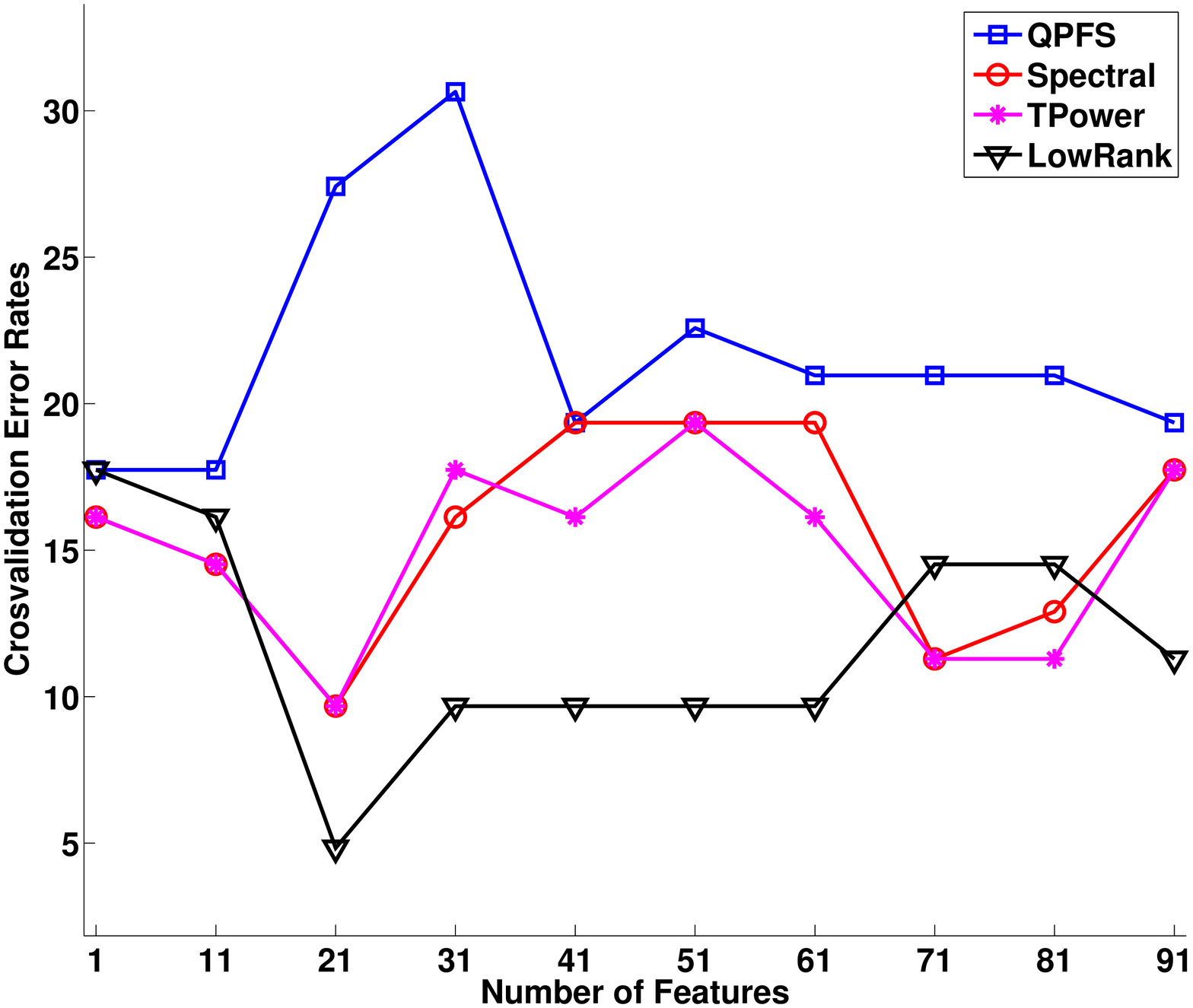}
                \caption{Colon}
                \label{fig:globalColon}
        \end{subfigure}%
        \quad 
        \begin{subfigure}[b]{0.3\textwidth}
                \includegraphics[trim = 5mm 0mm 3mm 7mm, clip, width=2.3in, height=1.5in]{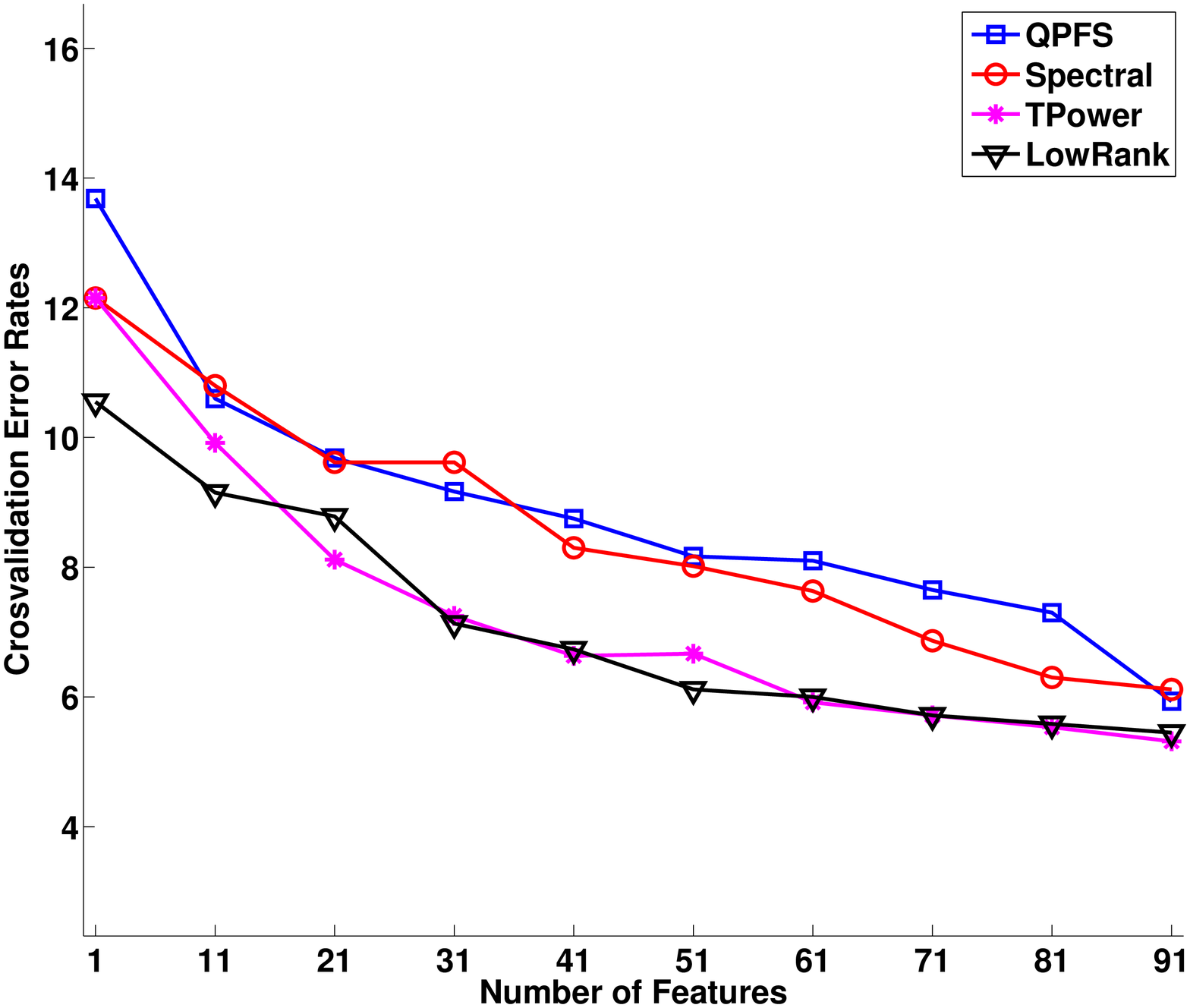}
                \caption{Gisette}
                \label{fig:globalGisette}
        \end{subfigure}
        \quad 
        \begin{subfigure}[b]{0.3\textwidth}
                \includegraphics[trim = 5mm 0mm 3mm 7mm, clip, width=2.3in, height=1.5in]{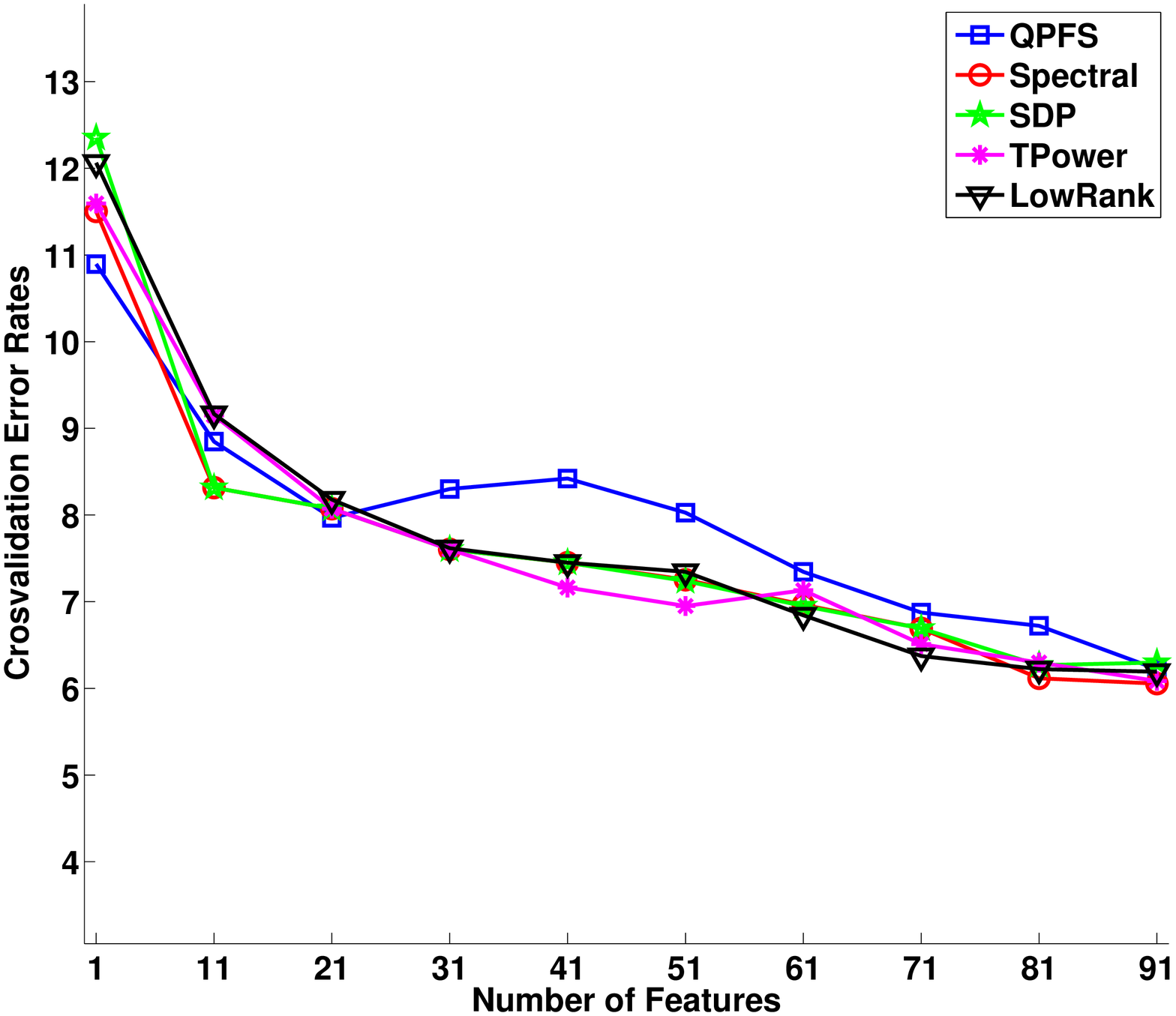}
                \caption{Musk2}
                \label{fig:globalMusk2}
        \end{subfigure}
        \caption{Average cross validation error(\%) vs. Number of features. First Row: Comparison of Greedy methods with \textbf{TPower} 
        and \textbf{LowRank} for 3 datasets. Second Row: Comparison of Global methods across 3 datasets.}
\label{fig:experiments}
\end{figure*}
\vspace*{-2mm}
\section{Discussion and Conclusions}
\vspace*{-2mm}
\label{discuss:sec}
Feature selection is a NP-hard problem and newer methods to approximate the solution 
will help drive research in this area. 
We have demonstrated that current methods applying MI and CMI for feature selection, inherently assume conditional independence between features. 
The conditional independence assumption limits the number of features compared to only 2 or 3 features at a time. 
There is need to derive better measures to approximate the importance of a group of selected features. 
To estimate the probability distributions, we had to discretize the features. We have not studied the 
effect of discretization in our work. Progress along all of these fronts 
will provide directions to improve MI based feature selection. 
In conclusion, we can state that both \textbf{TPower} and \textbf{LowRank} perform better than existing global and iterative techniques 
across most of the datasets. While \textbf{LowRank} slightly outperforms \textbf{TPower}, it does not compare well 
with regards to time.
\vspace*{-2mm}
\section{Acknowledgments}
This material is based upon work supported by the National Science Foundation (NSF) under Grant No:1116360. Any opinions, findings, and conclusions or recommendations expressed in this material are those of the authors and do not necessarily reflect the views of the NSF.\\
We express our thanks to Nguyen X. Vinh \cite{nguyen2014effective}, for providing his implementation for comparison.


%


\appendix
\section{\text{LP2} provides a Good Lower bound for \text{BQP} }
We would like to study the goodness of approximation to BQP provided by the solution to LP2. 
We define an equivalent problem in terms of BQP. 
\begin{Prop}
\begin{flalign*}
&q: \{0,1\}^n \rightarrow (-\infty,0], ~\text{where}~\\
&q(\bx) = \text{BQP}-||\bQ||_1
\end{flalign*}
is equilvalent to BQP.
\end{Prop}
\begin{proof}
$||\bQ||_1$ is the sum of all elements in $\bQ$. Since $\bQ_{ij} \geq 0$, BQP $\leq ||\bQ||_1~\forall\bx$. 
Therefore, $q(\bx)\leq 0, ~\forall \bx$. Since $||\bQ||_1$ is a constant for a matrix, under the same set of constraints,
\[
\argmax_{\mathbf{x}} \text{BQP} \equiv \argmax_{\mathbf{x}} q(\bx)\\
\]
\end{proof}
We define some new quantities for the derivation of the bound. Let $\bx^*$ be the solution of BQP. Let $\bar{\bx}$ be the solution of 
LP2. Since $||\bQ||_1 = \sum_{ij}\bQ_{ij}$, we can expand $||\bQ||_1$ in terms of any binary vector $\bx$. Specifically we define $||\bQ||_1,$ 
in terms of $\bar{\bx}$,
\begin{Def}
\begin{flalign}
&||\bQ||_1 = Q^0 + Q^1 + Q^2 ~ \text{where},\label{Qnorm:Eq}\\
&Q^0 \leftarrow \sum_{i,j|\bar{x}_i + \bar{x}_j=0}\bQ_{ij} \label{Q0:Eq}\\
&Q^1 \leftarrow \sum_{i,j|\bar{x}_i + \bar{x}_j=1}\bQ_{ij} \label{Q1:Eq} \\
&Q^2 \leftarrow \sum_{i,j|\bar{x}_i + \bar{x}_j=2}\bQ_{ij} \equiv \bar{\bx}^\top\bQ\bar{\bx} \label{Q2:Eq}
\end{flalign}
\end{Def}
\begin{Lemma}
\label{Lin:Lm1}
$||\bQ||_1 - \bar{\bx}^\top\bQ\bar{\bx} \geq Q^1$
\end{Lemma}
\begin{proof}
From (\ref{Qnorm:Eq}) we have,
\begin{flalign*}
||\bQ||_1 &= Q^0 + Q^1 + Q^2\\
||\bQ||_1 &\geq Q^1 + Q^2 \\
||\bQ||_1 & - \bar{\bx}^\top\bQ\bar{\bx} \geq Q^1 ~\text{using (\ref{Q2:Eq})}
\end{flalign*}
\end{proof}
Let $Q^*$ denote the maximum value of BQP and let $Q^*_{LP2}$ denote maximum value of LP2. 
If $\bx^*$ is the solution of BQP and $\bar{\bx}$ is the solution of LP2. 
We have the following result: 

\begin{Lemma}
$Q^*_{LP2} \geq Q^*$
\label{Lin:Lm2}
\end{Lemma}

\begin{proof}
\begin{flalign*}
Q^*_{LP2} &= \max||\bQ\bx||_1\\
& = ||\bQ\bar{\bx}||_1 \\
& \geq ||\bQ\bx^*||_1 \\
& \geq \bx^{*\top}\bQ\bx^*\\
& = Q^*
\end{flalign*}
\end{proof}
We are now ready to state the bound for LP2. 
\begin{Them}
\begin{flalign}
2q(\bx^*) \leq q(\bar{\bx}) \label{linbound:Eq}
\end{flalign}
\end{Them}
\begin{proof}
From Lemma (\ref{Lin:Lm2}), we have:
\begin{flalign}
\bx^{*\top}\bQ\bx^* &\leq \frac{1}{2}\sum_{ij}\bQ_{ij}(\bar{x}_i + \bar{x}_i)\\
\bx^{*\top}\bQ\bx^* &\leq \frac{1}{2}Q^1 + \bar{\bx}^\top\bQ\bar{\bx} \quad~~\text{(\ref{Q1:Eq}, \ref{Q2:Eq})}\\
2\bx^{*\top}\bQ\bx^* &\leq Q^1 + 2\bar{\bx}^\top\bQ\bar{\bx}\\
2\bx^{*\top}\bQ\bx^* &\leq ||\bQ||_1 + \bar{\bx}^\top\bQ\bar{\bx} \quad~~\text{Lemma (\ref{Lin:Lm1})} \\
2q(\bx^*) &\leq q(\bar{\bx}) \label{LinThmProof:Eq}
\end{flalign}
\end{proof}
The last statement (\ref{LinThmProof:Eq}) is arrived at by adding $-2||\bQ||_1$ on both sides. 
Since $q(\bx) \leq 0$, $2q(\bx^*) \leq q(\bar{\bx})$ implies that $q(\bar{\bx})$ is a lower bound for $q(\bx^*)$. 
We are therefore guaranteed a lower bound for QBP by solving LP2 and (\ref{LinThmProof:Eq}) provides the tightness of the bound.

\end{document}